\pdfoutput=1
\documentclass[11pt]{article}

\usepackage[]{acl}

\usepackage{times}
\usepackage{latexsym}

\usepackage{stmaryrd}
\usepackage{xfrac}

\usepackage{amssymb}
\usepackage{graphicx}

\usepackage{inconsolata}

\usepackage[T1]{fontenc}
\usepackage[utf8]{inputenc}

\usepackage{microtype}

\usepackage{amsmath}
\usepackage{amssymb}
\usepackage{amsthm}
\usepackage{amsfonts}
\usepackage{mathtools}
\usepackage{enumitem}
\usepackage{relsize}
\usepackage{bbm}
\usepackage{booktabs}
\usepackage{tabularx}

\usepackage{emoji}
\usepackage{inconsolata}

\usepackage{thm-restate}
\usepackage{apxproof}

\newtheorem{theorem}{Theorem}

\newtheorem{lemma}{Lemma}

\newtheorem{cor}{Corollary}

\usepackage[plain]{algorithm}
\usepackage{algorithmicx}
\usepackage[noend]{algpseudocode}
\algnewcommand{\parState}[1]{\State%
    \parbox[t]{\dimexpr\linewidth-\algmargin}{\strut\hangindent=\algorithmicindent \hangafter=1 #1\strut}}
\algrenewcommand\algorithmicindent{1.0em}%
\newcommand{\algorithmicdowhile}{\textbf{do}:}
\algdef{SE}[DOWHILE]{Do}{doWhile}{\algorithmicdowhile}[1]{\algorithmicwhile\ #1}%
\newcommand{\algorithmicfunc}[1]{\textbf{def} #1 :}
\algdef{SE}[FUNC]{Func}{EndFunc}[1]{\algorithmicfunc{#1}}{}

\newif\ifboldnumber

\algrenewcommand\alglinenumber[1]{%
  \footnotesize\ifboldnumber\color{red}\bfseries\fi\global\boldnumberfalse#1:}
\MakeRobust{\Call}   % fix algpseudocode bug with nested \Call commands

\newcommand{\rightcomment}[1]{{\color{gray} \(\triangleright\) {\footnotesize\textit{#1}}}}
\algrenewcommand{\algorithmiccomment}[1]{\hfill \rightcomment{#1}}  % redefines \Comment
\algnewcommand{\LineComment}[1]{\State \rightcomment{#1}}
\algnewcommand{\LinesComment}[1]{\State \rightcomment{\parbox[t]{\linewidth-\leftmargin-\widthof{\(\triangleright\) }}{#1}}}

\renewcommand\algorithmicthen{:}

\makeatletter
\ifthenelse{\equal{\ALG@noend}{t}}%
  {\algtext*{EndFunc}}
  {}%
\makeatother

\algnewcommand{\IIf}[1]{\State\algorithmicif\ #1\ \algorithmicthen}
\algnewcommand{\EndIIf}{\unskip}

\usepackage{cleveref}
\crefname{section}{\S}{\S\S}
\Crefname{section}{\S}{\S\S}
\crefname{table}{Tab.}{}
\crefname{figure}{Fig.}{}
\crefname{algorithm}{Alg}{}
\crefname{algorithm}{Alg}{}
\crefname{line}{Line}{}
\crefname{appendix}{App.}{}
\crefformat{section}{\S#2#1#3}
\crefname{theorem}{Theorem}{}
\crefname{proposition}{Proposition}{}
\crefname{definition}{Definition}{}
\crefname{lemma}{Lemma}{}
\crefname{cor}{Corollary}{}
\crefname{equation}{}{}

\usepackage[stable,hang,flushmargin]{footmisc}  % save space on footnotes.

\title{Exact \tikzmarknode{paired}{Paired}-\tikzmarknode{perm}{Permutation} Testing for Structured Test Statistics}
\newcommand{\ucambridge}{\emoji[emojis]{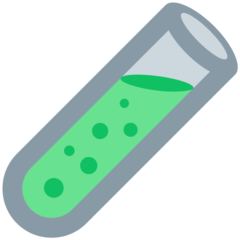}}
\newcommand{\ethz}{\emoji[emojis]{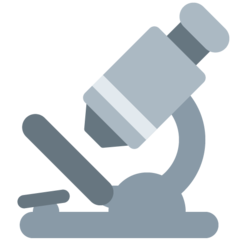}}
\newcommand{\jhu}{\emoji[emojis]{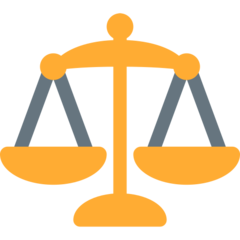}}

\author{
{\tikzmarknode{ranz}{Ran Zmigrod}\raise1.0ex\hbox{\normalfont\ucambridge}\raise1.0ex\hbox{\normalfont}}~\;~\tikzmarknode{timv}{Tim Vieira}\raise1.0ex\hbox{\normalfont\jhu}~\;~\tikzmarknode{ryan}{Ryan Cotterell}\raise1.0ex\hbox{\normalfont\ethz}
\\
  \raise1.0ex\hbox{\normalfont\ucambridge}\tikzmarknode{cam}{University of Cambridge}~\;~\raise1.0ex\hbox{\normalfont\jhu}\tikzmarknode{jhu}{Johns Hopkins University}~\;~\raise1.0ex\hbox{\normalfont\ethz}\tikzmarknode{eth}{ETH Z\"{u}rich} \\
  \href{mailto:rz279@cam.ac.uk}{\tt rz279@cam.ac.uk}~\;~\href{mailto:tim.f.vieira@gmail.com}{\tt tim.f.vieira@gmail.com} \\ \href{mailto:ryan.cotterell@inf.ethz.ch}{\tt ryan.cotterell@inf.ethz.ch}
}

\renewcommand{\ldots}{\ensuremath{{\ldotp\kern-0.2em\ldotp\kern-0.2em\ldotp}}}

\renewcommand{\cdots}{\ensuremath{{\cdotp\kern-0.2em\cdotp\kern-0.2em\cdotp}}}
\renewcommand{\dots}{\ensuremath{{\ldotp\kern-0.2em\ldotp\kern-0.2em\ldotp}}}

\newcommand{\checkNotation}[1]{{#1}}

\newcommand{\pvalue}{\checkNotation{p}}

\newcommand{\threshold}{\checkNotation{\alpha}}

\renewcommand{\ast}{\star}

\newcommand{\sys}[1]{\checkNotation{\mathit{#1}}}
\newcommand{\sysU}{\sys{U}}
\newcommand{\sysV}{\sys{V}}

\newcommand{\defn}[1]{\textbf{#1}}
\renewcommand{\th}[0]{^{\text{th}}}

\renewcommand{\setminus}[0]{\smallsetminus}
\newcommand{\defeq}[0]{\mathrel{\stackrel{\textnormal{\tiny def}}{=}}}

\newcommand{\bigo}[1]{\mathcal{O}(#1)}
\newcommand{\abs}[1]{\left\lvert #1 \right\rvert}
\newcommand{\case}[1]{\checkNotation{\noindent\emph{#1:}}}

\let\emptyset\varnothing

\newcommand{\tuple}[1]{\checkNotation{\!\left\langle #1 \right\rangle}}
\newcommand{\Set}[1]{\checkNotation{\{ #1 \}}}

\newcommand{\score}{\checkNotation{g}}
\newcommand{\hscore}{\checkNotation{h}}
\newcommand{\Hscore}[1]{\hscore\!\left(#1\right)}

\newcommand{\fone}{\checkNotation{F_1}}

\newcommand{\truep}[1]{\checkNotation{\mathrm{tp}\!\left( #1 \right)}}

\newcommand{\false}[1]{\checkNotation{\mathrm{in}\!\left( #1 \right)}}

\newcommand{\mat}[1]{\checkNotation{\mathbf{#1}}}

\newcommand{\zerovector}{\mat{0}}
\newcommand{\indicator}[1]{\checkNotation{{\mathbbm{1}\!\left[\normalcolor {#1}\right]}}}

\newcommand{\vouta}{\mat{{u}}}
\newcommand{\voutb}{\mat{{v}}}
\newcommand{\outan}{\checkNotation{{u}_n}}
\newcommand{\outbn}{\checkNotation{{v}_n}}

\newcommand{\Prob}[1]{\checkNotation{\mathbb{P}}\!\left[ #1 \right]}

\newcommand{\plusequal}{{\,\textsf{+=}\,}}

\newcommand{\nG}{\checkNotation{G}}
\newcommand{\nN}{\checkNotation{N}}

\newcommand{\nK}{\checkNotation{K}}

\newcommand{\real}{\checkNotation{\mathbb{R}}}

\newcommand{\algFace}[1]{\texttt{#1}}

\newcommand{\permtest}{\algFace{exact\_perm\_test}}
\newcommand{\permtestM}{\algFace{exact\_perm\_test}_{\algFace{m}}}

\newcommand{\montecarlo}{\algFace{monte\_carlo}}
\newcommand{\convolveSlow}{\algFace{convolve\_DP}}
\newcommand{\convolveFast}{\algFace{convolve\_FFT}}

\newcommand{\monstersum}[1]{{\sum\limits_{\mathclap{\substack{#1}}}}}  % smashed and possible stacked
\newcommand{\pmfn}[1]{\checkNotation{f_{#1}}}
\newcommand{\var}{\checkNotation{\xi}}

\newcommand{\obs}{\checkNotation{\overline{\var}}}

\newcommand{\convolve}{\mathop{\checkNotation{\ast}}}

\newcommand{\runtime}{\checkNotation{r}}
\newcommand{\spaceComp}{\checkNotation{s}}

\newcommand{\teststatistic}{\checkNotation{t}}

\newcommand{\randomvariable}[1]{\checkNotation{\mathrm{#1}}}
\newcommand{\RV}{\checkNotation{\randomvariable{Z}}}
\newcommand{\RVn}{\checkNotation{\randomvariable{Z}_n}}

\newcommand{\rvU}{\checkNotation{\randomvariable{\mathbf{U}}^{\emptyset}}}
\newcommand{\rvV}{\checkNotation{\randomvariable{\mathbf{V}}^{\emptyset}}}
\newcommand{\rvUn}{\checkNotation{\randomvariable{U}_n^{\emptyset}}}
\newcommand{\rvVn}{\checkNotation{\randomvariable{V}_n^{\emptyset}}}
\newcommand{\pmfvar}{\checkNotation{\mathrm{pmf}_{\RV}}}

\newcommand{\mF}[1]{\checkNotation{\mathrm{F}}_{#1}}
\newcommand{\domain}[1]{\checkNotation{\mathrm{dom}}\!\left( #1 \right)}

\newcommand{\varnStay}{\checkNotation{\overrightarrow{\var_n}}}
\newcommand{\varnSwap}{\checkNotation{\overleftarrow{\var_n}}}

\newcommand{\sumRV}{\randomvariable{S}}
\newcommand{\RVdom}{\checkNotation{\mathcal{S}}}

\usepackage{tikz}
\usetikzlibrary{calc,shapes}
\usetikzlibrary{tikzmark}
\usetikzlibrary{arrows}

\allowdisplaybreaks

\begin{document}
\maketitle

\begin{abstract}
Significance testing---especially the paired-permutation test---has played a vital role in developing NLP systems to provide confidence that the difference in performance between two systems (i.e., the test statistic) is not due to luck.
However, practitioners rely on Monte Carlo approximation to perform this test due to a lack of a suitable exact algorithm.
In this paper, we provide an efficient exact algorithm for the paired-permutation test for a family of structured test statistics.  Our algorithm runs in $\bigo{\nG\nN (\log \nG \nN) (\log \nN)}$ time where $\nN$ is the dataset size and $\nG$ is the range of the test statistic.
We found that our exact algorithm was $10$x faster than the Monte Carlo approximation with $20000$ samples on a common dataset.
\newline
\newline
\vspace{1.5em}
\hspace{.5em}\includegraphics[width=1.25em,height=1.25em]{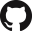}\hspace{.75em}\parbox{\dimexpr\linewidth-2\fboxsep-2\fboxrule}{\url{https://github.com/rycolab/paired-perm-test}}
\vspace{-.5em}
\end{abstract}

\section{Introduction}
How confident can we be that System $\sysU$ is more accurate than System $\sysV$? 
Questions of this form are widespread in natural language processing \citep{dietterich1998approximate, koehn-2004-statistical, OjalaG10, clark-etal-2011-better, berg-kirkpatrick12, dror-etal-2018-hitchhikers}
and statistical hypothesis testing provides answers \citep{lehman-book}. 
In this paper, we study the paired-permutation test \citep{good2000}---a commonly used hypothesis test in NLP because it makes no assumptions on the distribution of the data or the evaluation metric used to compare the two systems \citep{yeh-2000-accurate,dror-etal-2018-hitchhikers, dror2020, deutsch-etal-2021-statistical}.
The paired-permutation test checks whether a test statistic is significant by evaluating the probability that a value at least as large as the observed statistic would occur if system outputs were randomly swapped.
Thus, an exact algorithm for evaluating the paired-permutation test involves a summation over all $2^N$ possible swaps.
Without any assumptions on the test statistic, we can only exactly compute this sum in $\bigo{2^N}$ time.
Thus, practitioners often resort to running a Monte Carlo (MC) approximation which replaces the summation with $K \ll 2^N$ randomly sampled swaps.
Although the MC approximation is often practical, it unfortunately,introduces additional error when determining the significance of a test \citep{serlin2000testing, koehler2009assessment}.

\begin{figure}[t!]
\centering
\includegraphics[width=0.5\textwidth]{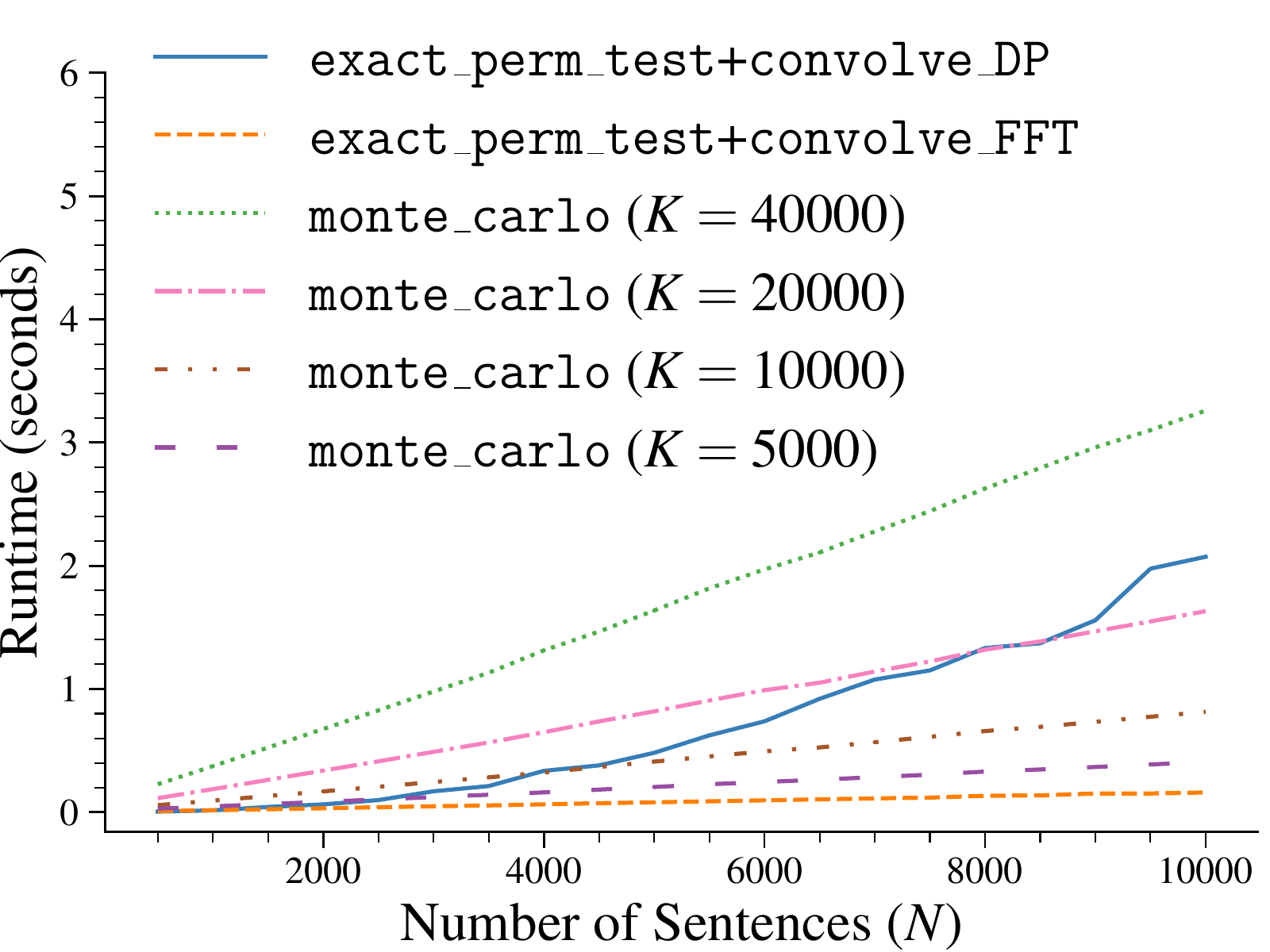}
\caption{Runtime comparison of $\permtest$ using $\convolveSlow$ and $\convolveFast$,  and $\montecarlo$ as a function of the number of entries in the dataset.
See for \cref{sec:experiment} for experimental details.\looseness=-1
}
\label{fig:runtime}
\end{figure}

This paper proposes a family of additively structured, integer-valued test statistics.
Test statistics of this form admit an efficient exact algorithm that leverages the fast Fourier transform \citep{cooley1965algorithm, cormen} to run in $\bigo{\nG\nN (\log \nG\nN) (\log \nN)}$ time 
where $\nN$ is the size of the dataset and $\nG$ is the range of the test statistic.
We compare the efficiency of our exact method to the MC approximation for comparing part-of-speech taggers on the Universal Dependency Dataset \citep{ud}.
Surprisingly, our \emph{exact} algorithm is faster than MC approximation: given $10000$ sentences, our algorithm is
$10$x faster than MC with $20000$ samples
and $3$x faster than MC with $5000$ samples,
taking $\approx$ 0.1 seconds.

\section{Paired-Permutation Testing}
The \defn{paired-permutation test} \citep{good2000}, the focus of this work, is a common null hypothesis
significance test that has a natural application to many problems in NLP \citep{peyrard20}. 
The test attempts to reject the null hypothesis, described below, at significance level $\threshold$; typically $\threshold = 0.05$.

\paragraph{Preliminaries.}
Suppose we want to compare the performances of two systems $\sysU$ and $\sysV$ 
where each system was evaluated on a dataset of the same $\nN$ entries.
We place the entries of $\sysU$ and $\sysV$ into a pair of arrays of length $\nN$ 
denoted $\vouta$ and $\voutb$.

\paragraph{The null hypothesis.}
The goal of a paired-permutation test is to test whether
the entries $\vouta$ and $\voutb$ are \emph{independent} of 
the labels $\sysU$ and $\sysV$ themselves.
The reason that this is the question we ought to care about is that, fundamentally, if the system label (in this case, $\sysU$ or $\sysV$) provides no information (in the sense of mutual information) about the entry, then we should \emph{not} prefer one system to another.
And, from basic information theory, we know that two random variables (RVs) have no mutual information iff they are independent.
So, independence of the system's label and the system's set of entries is the right thing to inquire about.
In the language of frequentist testing, 
the hypothesis that a system's labels and individual entries are independent is known as the \defn{null hypothesis}. And, under a paired-permutation test, the goal is to ascertain whether the data (the observed entries $\vouta$ and $\voutb)$ provide enough evidence to \emph{reject} the null hypothesis, i.e., to conclude that the label of a system shares information with the quality of its individual entries, and are indeed dependent.\looseness=-1

\paragraph{The null distribution.}
Next, in order to attempt to reject the null hypothesis, we require a distribution over (hypothetical) pairs of entries $\vouta'$ and $\voutb'$ whose individual entries are independent of the system labels $\sysU$ and $\sysV$, which is achieved
through the construction of RVs $\rvU$ and $\rvV$, whose joint distribution can be used to sample our hypothetical $\vouta'$ and $\voutb'$.
Traditionally, $\Prob{\rvU, \rvV}$ is referred to as the \defn{null distribution}.
A paired-permutation test provides a simple recipe for constructing such an RV pair.
The first step is to make an entry-wise independence assumption:
we define the joint probability $\Prob{\rvU, \rvV} \defeq \prod_{n=1}^N \Prob{\rvUn, \rvVn}$.
This means that the prediction a system makes for the $n\th$ entry is independent of the $m\th$ entry when $n\neq m$.
In the second step, we further define the entry-wise joint distribution as\looseness=-1
\begin{subequations}
\begin{align}
\displaystyle \Prob{\rvUn = \tikzmarknode{MarkD}{\outan}, \rvVn = \tikzmarknode{MarkA}{\outbn}} &\defeq \frac{1}{2} \; \quad\text{\color{gray}(stay)}\hspace{-10pt}\\
\displaystyle \Prob{\rvUn = \tikzmarknode{MarkB}{\outbn}, \rvVn = \tikzmarknode{MarkC}{\outan}} &\defeq \frac{1}{2} \; \quad\text{\color{gray}(swap)}\hspace{-10pt}
\end{align}
\end{subequations}
\begin{tikzpicture}[overlay,remember picture]
\draw[->,black!40, thick] 
  ([shift=({1.5mm,-1mm})]MarkD.south) -- ([shift=({-1.5mm,1mm})]MarkC.north); 
\draw[->,black!40, thick] 
  ([shift=({-1.5mm,-1mm})]MarkA.south) -- ([shift=({1.5mm,1mm})]MarkB.north);
\end{tikzpicture}
\!\!\!{\noindent}In words, $\Prob{\rvUn, \rvVn}$ is a uniform distribution over swapping $\sysU$ and $\sysV$'s prediction for the $n\th$ entry.
All in all, this definition of $\Prob{\rvU, \rvV}$ as the null distribution gives us a uniform distribution over all $2^{\nN}$ ways swapping of the labels and the individual entries of the observed predictions $\vouta$ and $\voutb$.
And, importantly, the joint distribution $\Prob{\rvU, \rvV}$, encodes the fact that the sampled entries are independent of the system label.\looseness=-1

\paragraph{The test statistic and the $\pvalue$-value.}
The final ingredient we need in a null hypothesis test is a \defn{test statistic}, whose job it is to provide a summary of samples $(\vouta', \voutb') \sim \Prob{\rvU, \rvV}$ and thereby facilitate comparison of samples from the null distribution $\Prob{\rvU, \rvV}$ and the observed entries $(\vouta, \voutb)$. 
In this work, we will define a test statistic as function $\teststatistic(\vouta, \voutb)$.
In principle, we can choose \emph{any}
test statistic $\teststatistic$ that allows us to distinguish $\vouta$ and $\voutb$, i.e., we have have $\teststatistic(\vouta, \voutb) = 0 \iff \vouta = \voutb$.
Now, given observed entries $\vouta$ and $\voutb$, the $\pvalue$-value is defined as\looseness=-1
\begin{equation}
    \pvalue = \Prob{\teststatistic(\rvU, \rvV) \geq \obs}
\end{equation}
where $\obs\defeq\teststatistic(\vouta, \voutb)$ is the \defn{observed effect}.
In words, the $\pvalue$-value is the probability of observing a test statistic $\teststatistic(\vouta', \voutb')$ with a value as large as $\teststatistic(\vouta, \voutb)$ where $(\vouta', \voutb') \sim \Prob{\rvU, \rvV}$ are sampled from the null distribution.
Recall that the system labels and entries are independent under the null distribution
by construction, so the $\pvalue$-value tells us, under the independence assumption, how likely such a large test statistic would have been observed.
The test says that we have sufficient evidence to reject the null hypothesis when $\pvalue < \threshold$. 
These concepts are depicted in \cref{fig:pmf}.

\begin{figure}[h]
\begin{tikzpicture}[y=4cm,
bar/.pic={
    \fill (-.1,0) rectangle (.1,#1) (0,#1) node[above,scale=1/2]{};
  }
]
\draw
% the main axis
(0,0) edge[black!30] (6,0)
(0,0) edge[black!30,midway] node[xshift=-1em] {{\footnotesize $\Prob{\teststatistic}$}} (0,6/16)
% draw the distribution and label it
foreach[count=\i] ~ in {.2/16,.5/16,1/16,3/16,5/16,6/16,5/16,3/16,1/16,.5/16,.2/16}{
    (\i/2,0) pic[blue!30!gray]{bar=~} node[below]{}
}
foreach[count=\i] ~ in {.2/16,.5/16,1/16,3/16,5/16,6/16,5/16,3/16}{
    (\i/2,0) pic[gray!20]{bar=~} node[below]{}
};
\node[black!40] at (4.25,-.75em) {{\footnotesize $\obs$}};
\draw (4.25,0) edge[black,dotted] (4.25,7/16);
\node[black!30] at (6.2,0) {$\teststatistic$};
\node at (5.4,6/16) {$\pvalue = \sum_{t \ge \obs} {\tikz \fill[blue!30!gray] (0,0) rectangle (0.2,0.1);}_t$};
\end{tikzpicture}
\caption{Depiction of the pmf of the test statistic $t$, the $\pvalue$-value, and the observed value $\obs$.}
\label{fig:pmf}
\end{figure}
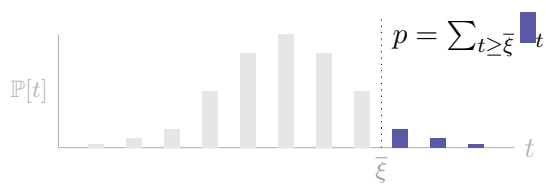

\section{Structured Test Statistics}
We now discuss a common special case of the paired-permutation test where the test statistic has a particular structure.
In \cref{sec:algs}, we show how to exploit this structure to develop an efficient algorithm to exactly compute the test.
The specific assumption we make is that the test statistic is an integer-valued \defn{additively decomposable  function}.
Formally, this assumption means that we can rewrite $\teststatistic$ as follows\looseness=-1
\begin{equation}\label{eq:subtraction}
    \teststatistic(\vouta, \voutb) 
    \defeq \hscore\!\left(\boldsymbol{\score}(\vouta, \voutb)\right)
\end{equation}
for any function $\hscore$ and additively decomposable function $\boldsymbol{\score}(\vouta, \voutb) \!\defeq\! \sum_{n=1}^\nN \score(\outan, \outbn)$ such that $\score$ is an integer-valued function with a range of size $\bigo{\nG}$.
The structure of \cref{eq:subtraction} will allow us to derive an efficient algorithm for evaluating $\Prob{\teststatistic(\rvU, \rvV)} \!=\! \Prob{\hscore\!\left(\sum_{n=1}^\nN \score(\outan, \outbn)\right)}$. 
We now dissect this equation.
Each summand $\score(\outan, \outbn)$ can take on one of two values 
$\varnStay\!\defeq\!\score(\outan, \outbn)$ and $\varnSwap\!\defeq\!\score(\outbn, \outan)$ 
with equal probability.  We rewrite the sum $\sum_{n=1}^\nN \score(\outan, \outbn)$ as
$\sumRV\!\defeq\!\sum_{n=1}^\nN \RVn$ where $\RVn$ are uniform RVs over the set $\Set{\varnStay, \varnSwap}$. Each $\RVn$ has  probability mass function (PMF)
\begin{subequations}
\begin{align}
   \pmfn{n}(z) &\defeq \pmfvar(\varnStay, \varnSwap)(z) \\
   &\defeq \label{eq:pmf}
   \begin{cases}
        \frac{1}{2}\indicator{z \in\Set{\varnStay, \varnSwap}} & \textbf{if } \varnStay \ne \varnSwap \\
        \indicator{z = \varnStay} & \textbf{otherwise}
    \end{cases}
\end{align}
\end{subequations}

\noindent The domain of each PMF, $\domain{\pmfn{n}}$, contains at most two elements.
Let $\RVdom\!\defeq\!\domain{\sumRV}$.
Clearly, $\abs{\RVdom} = \bigo{\nG \nN}$ as we have a sum over $\nN$ RVs $\RVn$ each with domain size $\bigo{\nG}$.  The following theorem shows that we can evaluate 
$\Prob{\teststatistic(\rvU, \rvV)}$ from the distribution of $\sumRV$, which we we will show in the next section is efficient to compute.

\begin{theorem}\label{thm:statistic}
For any test statistic $\teststatistic$ that factorizes as in \cref{eq:subtraction} with $\hscore$ and $\score$,
the distribution of the test statistic under the null distribution decomposes as\looseness=-1
\begin{equation}
    \Prob{\teststatistic(\rvU, \rvV)}
    = \Prob{\hscore(\sumRV)} 
\end{equation}
\end{theorem}
\begin{proof}
\begin{subequations}
\begin{align}
\!\! \Prob{ \teststatistic(\rvU, \rvV) }  
    &=  \Prob{\hscore\!\left(\sum_{n=1}^{\nN} \score(\rvUn, \rvVn)\right)}  \\
    &=  \Prob{ \hscore\!\left(\sum_{n=1}^{\nN}  \RVn \right)} \\
    &= \Prob{\hscore(\sumRV)} 
\end{align}
\end{subequations}
\end{proof}

\paragraph{Example.}
A common test statistic is the difference in accuracy, in which each entry $\outan\in\Set{1,\dots,C}^\nG$ where 
$C$ is the number of classes and in this case, $\nG$ is the maximum length of an entry sequence (or one if each entry has a binary accuracy value).
Then $\score(\outan, \outbn)\in\Set{-\nG,\dots,\nG}$ is the difference in the number of correct predictions between individual entries $\outan$ and $\outbn$.
We can additionally define the function $\hscore$ as either $\hscore(x)=x$ or $\hscore(x)=\abs{x}$ depending on whether we want a one-tailed or two-tailed significance test.

\begin{algorithm}[t]
    \centering
    \begin{algorithmic}[1]
    \Func{$\montecarlo(\vouta, \voutb, \score, \hscore, \nK)$}
        \For{$n=1 \textbf{ to } \nN$}
        \State $\varnStay \gets \score(\tikzmarknode{AA}{\outan}, \tikzmarknode{CC}{\outbn})$ \Comment{Local effect (stay)} \vspace{2pt}
        \State $\varnSwap \gets \score(\tikzmarknode{DD}{\outbn}, \tikzmarknode{BB}{\outan})$ \Comment{Local effect (swap)}
        \State $\pmfn{n} \gets \pmfvar\left(\varnStay, \varnSwap \right)$  % \Comment{pmfs}
    \EndFor
\begin{tikzpicture}[overlay,remember picture]
\draw[->,black!40, thick] 
  ([shift=({1.5mm,-1mm})]CC.south) -- ([shift=({-1.5mm,1mm})]DD.north); 
\draw[->,black!40, thick] 
  ([shift=({-1.5mm,-1mm})]AA.south) -- ([shift=({1.5mm,1mm})]BB.north);
\end{tikzpicture}
    \State $\obs \gets \Hscore{{\displaystyle\sum_{n=1}^{\nN} \varnStay}}$  \Comment{Compute observed effect} 
    \LinesComment{Sample $K$ random \emph{stay} or \emph{swap} actions from each local pmf $\pmfn{n}$.}
    \State $z^{(k)}_n \sim \pmfn{n}$  \ \ \textbf{for} $n = 1 \textbf{ to } \nN, k = 1 \textbf{ to } \nK$
    \State \Return ${\displaystyle \frac{1}{\nK} \sum_{k=1}^\nK \indicator{\,  \Hscore{ \sum_{n=1}^\nN z_n^{(k)} }  \ge \obs\,}}$
    \EndFunc
    \end{algorithmic}
    \caption{Monte Carlo approximation algorithm for the paired-permutation test.}
    \label{alg:mc}
\end{algorithm}

\paragraph{A Monte Carlo paired-permutation test.}\label{sec:mc}
To the best of our knowledge, no practical \emph{exact} algorithm for the paired-permutation test has been given in the literature.
Thus, most practical implementations of the paired-permutation test use an MC
approximation, whereby one randomly samples from $\sumRV$ to approximate $\Prob{\rvU, \rvV}$.
We give this MC algorithm as $\montecarlo$ in \cref{alg:mc} which runs in $\bigo{\nK \nN}$ time where $\nK$ is the number of samples taken.\looseness=-1

\section{An Exact Paired-Permutation Test}\label{sec:algs}
In this section, we describe two exact, efficient algorithms for computing the $\pvalue$-value under the paired-permutation test for any structured test statistic (see \cref{eq:subtraction}).\footnote{In \cref{app:doad}, we extend our algorithms to work with test statistics that use $m$ additively decomposable scoring functions, e.g., difference in $\fone$ scores.}
Our algorithms hinge on an important theorem in probability: The PMF of the sum of independent events is the convolution of their individual PMFs \citep[p. 252]{ross}.
Let $\pmfn{\sumRV}$ denote the PMF of $\sumRV$. Since RVs $\RVn$ are independent, we have that
\begin{subequations}
\begin{align}
    &\hspace{-15pt}\Prob{\hscore(\sumRV)\ge\obs} \\
    = &\sum_{\var\in\RVdom} \pmfn{\sumRV}(\var) \,\indicator{\hscore(\var))\ge\obs} \\
    = &\sum_{\var\in\RVdom} (\pmfn{1}\convolve\cdots\convolve\pmfn{\nN})(\var)\, \indicator{\hscore(\var))\ge\obs} \label{eq:convolve-breakdown}
\end{align}
\end{subequations}
where $\convolve$ is the discrete \defn{convolution operator}.
For functions $\pmfn{i},\pmfn{j} \in \RVdom \to \real$, $\pmfn{i} \convolve \pmfn{j} \in \RVdom \to \real$ is given by the following expression
\begin{equation}
(\pmfn{i} \ast \pmfn{j})(\var) \defeq \sum_{\var' \in \RVdom} \pmfn{i}(\var') \, \pmfn{j}(\var - \var')
\end{equation}
Pseudocode for this algorithm is given as $\permtest$ in \cref{alg:permtest}. We omit the details of evaluating the convolution in $\permtest$ and discuss methods for efficient convolution in the remainder of this section.

\begin{algorithm}[t]
    \centering
    \begin{algorithmic}[1]
    \Func{$\permtest(\vouta, \voutb, \score, \hscore)$}
    
    \For{$n=1 \textbf{ to } \nN$}
        \State $\varnStay \gets \score(\tikzmarknode{AA}{\outan}, \tikzmarknode{CC}{\outbn})$ \Comment{Local effect (stay)} \vspace{2pt}
        \State $\varnSwap \gets \score(\tikzmarknode{DD}{\outbn}, \tikzmarknode{BB}{\outan})$ \Comment{Local effect (swap)}
        \State $\pmfn{n} \gets \pmfvar\left(\varnStay, \varnSwap \right)$  % \Comment{pmfs}
    \EndFor
\begin{tikzpicture}[overlay,remember picture]
\draw[->,black!40, thick] 
  ([shift=({1.5mm,-1mm})]CC.south) -- ([shift=({-1.5mm,1mm})]DD.north); 
\draw[->,black!40, thick] 
  ([shift=({-1.5mm,-1mm})]AA.south) -- ([shift=({1.5mm,1mm})]BB.north);
\end{tikzpicture}

    \State $\obs \gets \Hscore{{\displaystyle\sum_{n=1}^{\nN} \varnStay}}$  \Comment{Compute observed effect}  
    \State $\pmfn{\sumRV} \gets \pmfn{1} \ast  \cdots \ast \pmfn{\nN}$ \Comment{Convolve the $f_n$'s} \label{line:convolve}
    \LineComment{Sum-up the pmf to get \pvalue}
    \State \Return ${\displaystyle\monstersum{ \var \in \RVdom} \pmfn{\sumRV}(\var) \,\indicator{\Hscore{\var} \ge \obs\,}}$  \label{line:p-accum}
    \EndFunc
\end{algorithmic}
\caption{Compute the exact $\pvalue$ value for the paired-permutation test for structured test statistics.}
\label{alg:permtest}
\end{algorithm}

\begin{theorem}\label{thm:exact}
For any two entries, $\vouta$ and $\voutb$, 
and test statistic $\teststatistic$ that factorizes as in \cref{eq:subtraction} with $\hscore$ and $\score$, $\permtest(\vouta, \voutb, \score, \hscore)$ returns $\pvalue$ in $\bigo{\nG\nN {+} \runtime(\nG, \nN)}$ time, $\bigo{\nN {+} \spaceComp(\nG, \nN)}$ space. We define $\runtime(\nG, \nN)$ and $\spaceComp(\nG, \nN)$ as the time and space complexities for constructing $\pmfn{1}\convolve\cdots\convolve\pmfn{\nN}$.\looseness=-1
\end{theorem}
\begin{proof}
The correctness of $\permtest$ is by \cref{thm:statistic} and \cref{eq:convolve-breakdown}.
All lines except for \cref{line:convolve} and \cref{line:p-accum} require at most $\bigo{\nN}$ time and space.
\cref{line:p-accum} runs in $\bigo{\nG\nN}$ time and $\bigo{1}$ space.
Thus, $\permtest$ runs in $\bigo{\nN + \nG\nN + \runtime(\nG, \nN)}$ time and $\bigo{\nN + \spaceComp(\nG, \nN)}$ space.
\end{proof}

The computational question is then: What is the most efficient algorithm for evaluating $(\pmfn{1}\convolve\cdots\convolve\pmfn{\nN})(\var)$?
In the following two subsections, we present $\bigo{\nG\nN^2}$ time and $\bigo{\nG (\log \nG \nN) (\log \nN)}$ time algorithms for performing this $\nN$-fold convolution.

\subsection{Convolution by Dynamic Programming}
Our first approach builds a dynamic program (DP) that takes advantage of the sparsity of our RVs to efficiently construct the PMF $\pmfn{\sumRV}$.
We do this by constructing a PMF array $\mF{n}\in\real^{\RVdom}$ for $n\in\Set{0,\dots,\nN}$ (we use $n=0$ as an initialisation base case) such that $\mF{n}(\var) = (\pmfn{1} \convolve \cdots \convolve \pmfn{n})(\var)$.
As we apply each convolution, we know that $\pmfn{n}$ is only non-zero at $\varnStay$ and $\varnSwap$, and so we can run each convolution in $\bigo{\nG\nN}$ time.
The pseudocode for this approach is given as $\convolveSlow$ in \cref{alg:convolve-slow}.

\begin{algorithm}[t]
   \centering
    \begin{algorithmic}[1]
    \Func{$\convolveSlow(\pmfn{1},\dots,\pmfn{\nN})$}
    \State $\mF{} \gets \zerovector$
    \State $\mF{0}(0) \gets 1$  \Comment{Init: $\Prob{0}\!=\!1$ if $\nN\!=\!0$}
    \For{$n=1 \textbf{ to } \nN$}
    \LinesComment{Compute $\mF{n} = \pmfn{n} \star \mF{n-1}$}
    \For{$\var \in \domain{\pmfn{n}}$} \Comment{$\abs{\domain{\pmfn{n}}}\le 2$}
    \For{$\var' \in \domain{\mF{n-1}}$} \label{line:for-start}
    \State $\mF{n}(\var + \var') \plusequal \pmfn{n}(\var) \cdot \mF{n-1}(\var')$ \label{line:for-end} 
    \EndFor
    \EndFor
    \EndFor
    \State \Return $\mF{\nN}$
    \EndFunc
    \end{algorithmic}
    \caption{Dynamic programming algorithm to compute the pmf of $\sumRV$ 
    as the $\nN$-fold convolution
    of the pmfs $\pmfn{1},\dots,\pmfn{\nN}$.
    Note that we only need to store $\mF{n}$ and $\mF{n-1}$ at any given time.}
    \label{alg:convolve-slow}
\end{algorithm}

\begin{theorem}
For any RVs $\RV_1,\dots,\RV_{\nN}$ with PMFs $\pmfn{1},\dots,\pmfn{\nN}$, $\convolveSlow(\pmfn{n},\dots,\pmfn{\nN})$ returns $\pmfn{1} \convolve \cdots \convolve \pmfn{n}$ in $\bigo{\nG\nN^2}$ time, $\bigo{\nG \nN}$ space.
\end{theorem}
\begin{proof}
The proof of correctness of $\convolveSlow$ is given in \cref{app:dp}.
Each $\mF{n}$ has $\bigo{\nG \nN}$ elements and so $\convolveSlow$ clearly runs in $\bigo{\nG \nN^2}$ time.
Furthermore, at any iteration, we only require $\mF{n}$ and $\mF{n-1}$ and so we can execute $\convolveSlow$ in $\bigo{\nG \nN}$ space.
\end{proof}

\subsection{Convolution by FFT}
Our second approach uses the fast Fourier transform \citep[FFT;][]{cooley1965algorithm, cormen} to evaluate the convolutions.
Using this method means that each convolution takes $\bigo{\nG \nN \log \nG \nN}$ time $\bigo{\nG \nN}$ space.
We further exploit the commutativity of convolution to perform the $\nN$-fold convolution in $\log \nN$ convolutions using a recursive program.
The pseudocode for this approach is given as $\convolveFast$ in \cref{alg:convolve-fast}.

\begin{algorithm}[t]
    \centering
    \begin{algorithmic}[1]
    \Func{$\convolveFast(\pmfn{1},\dots,\pmfn{\nN})$}
    \If{$\nN = 1$}  \Return $\pmfn{1}$
    \EndIf
    \State \Return $\convolveFast(\pmfn{1},\dots,\pmfn{\nN \sslash 2})$
    \Statex \hfill $\convolve \convolveFast(\pmfn{\nN \sslash 2 + 1},\dots,\pmfn{\nN}) $
    \EndFunc
    \end{algorithmic}
    \caption{Recursive algorithm to compute the pmf of $\sumRV$ 
    as the $\nN$-fold convolution
    of the pmfs $\pmfn{1},\dots,\pmfn{\nN}$.
    Here the convolution operation ($\convolve$) runs in $\bigo{\nG\nN \log\nG\nN}$ time
    thanks to the FFT \citep{cooley1965algorithm}.
    }
    \label{alg:convolve-fast}
\end{algorithm}

\begin{theorem}\label{thm:convolve-fast}
For any RVs $\RV_1,\dots,\RV_{\nN}$ with PMFs $\pmfn{1},\dots,\pmfn{\nN}$,
$\convolveFast(\pmfn{1},\dots,\pmfn{\nN})$ returns $\pmfn{1} \convolve \cdots \convolve \pmfn{n}$ in $\bigo{\nG\nN (\log \nG\nN) (\log \nN)}$ time, $\bigo{\nG \nN \log \nN}$ space.\footnote{We note that the $\log \nN$ factor in the space complexity may be eliminated by tail recursion elimination \citep{muchnick}.}\looseness=-1
\end{theorem}
\begin{proof}
The correctness of $\convolveFast$ is due to \citet{cooley1965algorithm}.
The recursion of $\convolveFast$ can be given as
\begin{equation}
    T(\nN) = 2\,T\!\left(\frac{\nN}{2}\right) + \bigo{\nG\nN \log \nG\nN}
\end{equation}
Solving this recursion, we call $T$ $\bigo{\log\nN}$ times.
Therefore, the time complexity of $\convolveFast$ is $\bigo{\nG\nN (\log \nG\nN) (\log \nN)}$.
Similarly, each call requires $\bigo{\nG\nN}$ space and $\convolveFast$ has a space complexity of $\bigo{\nG\nN \log \nN}$.
\end{proof}

\begin{cor}
For any two entries, $\vouta$ and $\voutb$, 
and test statistic $\teststatistic$ that factorizes as in \cref{eq:subtraction} with $\hscore$ and $\score$, $\permtest(\vouta, \voutb, \score, \hscore)$ returns $\pvalue$ in $\bigo{\nG\nN (\log \nG\nN) (\log \nN)}$ time, $\bigo{\nG \nN}$ space.\looseness=-1
\end{cor}
\begin{proof}
The correctness and complexity bounds are due to \cref{thm:exact} and \cref{thm:convolve-fast}.
Specifically, \cref{line:convolve} can be executed using $\convolveFast$.
\end{proof}

\section{Experiments}\label{sec:experiment}
We demonstrate the efficiency of our exact algorithms by simulating paired-permutation tests between the accuracy of two systems.
In order to have some control over the $p$-value, $\nN$, and $\nG$ (maximum length of a sentence), we randomly generate our two system outputs from a measured distribution.
Specifically, we will use the Stanza\footnote{The code and pre-trained model are both freely accessible at \url{https://github.com/stanfordnlp/stanza}.} \citep{stanza} part-of-speech tag accuracy statistics when evaluating on the English Universal Dependencies (UD) test set \citep{ud}.
We sample our outputs from the normal distribution where the mean and standard deviation match the rates of Stanza's observed accuracy.
We further sample the length of each sample sentence according to the distribution of lengths in the test set.
These distributions are provided in \cref{tab:distributions}.

We show that, empirically, the exact test is \emph{more efficient} than the MC approximation; this is evinced in \cref{fig:runtime}
where we have compared the runtime of $\permtest$ using $\convolveSlow$ and $\convolveFast$ against $\montecarlo$ for various sample sizes ($\nK\!\in\!\Set{5000, 10000, 20000, 40000}$).$^,$\footnote{The experiment used an Apple M1 Max processor.}
We note that using $\convolveSlow$ is already more efficient than running $\montecarlo$ with $\nK=40000$ and $\nK=20000$ (up to $\nN\approx 8000)$.\footnote{We note that the average test set size of the $129$ UD treebanks we examined is just over $1000$ sentences, and only three treebanks had more than $6000$ sentences. 
% These were Czech, Japanese, and Russian which had test set sizes of roughly $10000$, $8000$, and $6500$ respectively.
% Additionally, the standard split of the commonly used Penn treebank (PTB) \citep{ptb} provides a test set of about $5500$ sentences.
}
Furthermore, $\convolveFast$ is \emph{much} faster and we observe a speed-up between $3$x and $30$x, depending on the number of samples $\nK$.
Indeed, using $\convolveFast$ allows us to perform an exact paired-permutation test for $\nN=10000$ in approximately one-tenth of a second.

\begin{table}[t]
    \centering
    \begin{tabularx}{0.47\textwidth}{lcc}
        \bf Metric & \bf Mean & \bf Standard Dev. \\ \midrule
        Accuracy & $0.9543$ & $0.1116$ \\
        Sentence length & $12.08$ & $10.60$
    \end{tabularx}
    \caption{Distributions for of accuracy and sentence length for POS tagging using Stanza \citep{stanza} on the English UD test dataset \cite{ud}.
    }
    \label{tab:distributions}
\end{table}

\section{Conclusion}
We presented an algorithm to compute the exact $\pvalue$-value of a paired-permutation test for the case of a family of structured test statistics, including the difference in accuracy.
Our algorithm runs in $\bigo{\nG\nN (\log \nG\nN) (\log \nN)}$ time and requires $\bigo{\nG\nN}$ space.
We empirically show that our exact algorithm is \emph{faster} than Monte Carlo approximation techniques.
The theory of our work is extensible to a more general class of test statistics which we discuss in \cref{app:doad}.
We hope that this work encourages the use of exact paired-permutation tests in future NLP research.

\section*{Ethical Concerns}
We foresee no ethical concerns in this work.

\section*{Acknowledgments}
We would like to thank the reviewers for their invaluable feedback and time spent engaging with our work. The first author is supported by the University of Cambridge School of Technology Vice-Chancellor's Scholarship as well as by the University of Cambridge Department of Computer Science and Technology's EPSRC.

\bibliography{acl}
\bibliographystyle{acl_natbib}

\clearpage
\onecolumn

\appendix

\section{Proof of Correctness of $\convolveSlow$}\label{app:dp}
We prove the correctness of $\convolveSlow$ using the following lemma.

\begin{lemma}
For any $\nN$ RVs $\RV_1,\dots,\RV_{\nN}$ with PMFs $\pmfn{1},\dots,\pmfn{\nN}$ respectively and $n\in\Set{1,\dots,\nN}$, $\convolveSlow(\pmfn{1},\dots,\pmfn{\nN})$ constructs $\mF{n}$ such that for any $\var\in\RVdom$,
\begin{equation}\label{eq:W}
    \mF{n}(\var) = \pmfn{:n}(\var) \defeq (\pmfn{1} \convolve \cdots \convolve \pmfn{n})(\var)
\end{equation}
\end{lemma}
\begin{proof}
We prove this by induction on $\nN$.

\case{Base case} $\nN=1$.
We have that $\mF{0}(0)=1$ and $\mF{0}(\var)=0$ for all $\var\in\RVdom\setminus\Set{0}$.
Therefore, $\mF{1}(\varnStay)=\mF{1}(\varnSwap)=\frac{1}{2}$ and $\mF{1}(\var)=0$ for all $\var\in\RVdom\setminus\Set{\varnStay, \varnSwap}$ as expected.

\case{Inductive step}
Assume \cref{eq:W} holds for $\nN=n-1$.
Let $\nN=n$ and consider $\pmfn{:(n-1)}\convolve \pmfn{n}$.
\begin{equation}
    (\pmfn{:(n-1)}\convolve \pmfn{n})(\var) 
    = \monstersum{\var'\in\RVdom}\pmfn{:(n-1)}(\var')\pmfn{n}(\var-\var')
    = \monstersum{\var'\in\domain{\pmfn{n}}}\pmfn{:(n-1)}(\var + \var')\pmfn{n}(\var')
    = \monstersum{\var'\in\domain{\pmfn{n}}}\mF{n-1}(\var + \var')\pmfn{n}(\var')
\end{equation}
This is exactly the construction in the for-loop between \cref{line:for-start} and \cref{line:for-end}.
Therefore, $\mF{n}(\var) = \pmfn{:n}(\var)$.
\end{proof}

As each $\mF{\nN}$ will contain the $\nN$-fold convolution $\pmfn{1} \convolve \cdots \convolve \pmfn{n}$, $\convolveSlow$ is correct by definition.

\section{Paired-Permutation Test for Higher-order Test Statistics}\label{app:doad}
In this section, we extend our approach for the paired-permutation test to test statistics that are functions of $m$ additively
decomposable functions.
In symbols, this assumption means that we can rewrite $\teststatistic$ as follows
\begin{equation}\label{eq:higher}
    \teststatistic(\vouta,\voutb) \defeq \hscore\left(\boldsymbol{\score_1}(\vouta, \voutb),\dots,\boldsymbol{\score_m}(\vouta, \voutb) \right)
\end{equation}
for any function $\hscore$ and integer-valued, additive decomposable functions $\boldsymbol{\score_i}$ (i.e., $\boldsymbol{\score_i}(\vouta, \voutb)\defeq\sum_{n=1}^{\nN}\score_i(\outan, \outbn)$.
We now define $\varnStay$ and $\varnSwap$ as $m$-tuples,
\begin{align}
    \varnStay &\defeq \tuple{\boldsymbol{\score_1}(\vouta, \voutb),\dots,\boldsymbol{\score_m}(\vouta, \voutb)} \\
    \varnSwap &\defeq \tuple{\boldsymbol{\score_1}(\voutb, \vouta),\dots,\boldsymbol{\score_m}(\voutb, \vouta)}
\end{align}
And so each RV $\RVn$ has the same PMF as in \cref{eq:pmf}.
We can then define an analogous function to $\permtest$ for the case of $m$ additively decomposable functions.
We give pseudocode for this as $\permtestM$ in \cref{alg:permtest-general}.
The convolution algorithms,$\convolveSlow$ and $\convolveFast$, can both be used to perform for convolution step in \cref{line:higher:convolve}.

\begin{theorem}\label{thm:exact-higher}
For any two entries, $\vouta$ and $\voutb$, 
and test statistic $\teststatistic$ that factorizes as in \cref{eq:higher} with $\hscore$ and $\score_1$ to $\score_m$, $\permtestM(\vouta, \voutb, \score_1,\dots,\score_m, \hscore)$ returns $\pvalue$ in $\bigo{\nG^m\nN^m (\log \nG\nN) (\log \nN)}$ time and $\bigo{\nG^m \nN^m \log \nN}$ space.\looseness=-1
\end{theorem}
\begin{proof}
The proof of correctness for $\permtestM$ is the same as \cref{thm:exact}.
The expensive operation in the algorithm is the convolution step (\cref{line:higher:convolve}). 
We can perform a FFT $m$-dimensional convolution in $\bigo{\nG^m \nN^m \log \nG \nN}$ time and $\bigo{\nG^4 \nN^4}$ space.
As we require $\bigo{\log \nN}$ convolution steps, $\permtestM$ runs in $\bigo{\nG^m\nN^m (\log \nG\nN) (\log \nN)}$ time and $\bigo{\nG^m \nN^m \log \nN}$ space.
\end{proof}

\begin{algorithm}[t!]
    \centering
    \begin{algorithmic}[1]
    \Func{$\permtestM(\vouta, \voutb, \score_1,\dots,\score_m, \hscore)$}
    
    \For{$n=1 \textbf{ to } \nN$}
        \State $\varnStay \gets \tuple{\score_1(\outan, \outbn), \ldots, \score_m(\outan, \outbn)}$ \Comment{Local effect (stay)} \vspace{2pt}
        \State $\varnSwap \gets \tuple{\score_1(\outbn, \outan), \ldots, \score_m(\outbn, \outan)}$ \Comment{Local effect (swap)}
        \State $\pmfn{n} \gets \pmfvar\left(\varnStay, \varnSwap \right)$ 
    \EndFor
   
    \State $\obs \gets \Hscore{\boldsymbol{\score_1}(\vouta, \voutb),\dots,\boldsymbol{\score_2}(\vouta,\voutb)}$  \Comment{Compute observed effect}  
    \State $\pmfn{\sumRV} \gets \pmfn{1} \ast  \cdots \ast \pmfn{\nN}$ \Comment{Convolve the $f_n$'s} \label{line:higher:convolve}
    \LineComment{Sum-up the pmf to get \pvalue}
    \State \Return ${\displaystyle \sum_{\var \in \RVdom} \pmfn{\sumRV}(\var) \,\indicator{\Hscore{\var} \ge \obs\,}}$  \label{line:higher:p-accum}
    \EndFunc
\end{algorithmic}
\caption{Pseudocode to find exact $\pvalue$ value for the paired-permutation test for test statistics comprised of $m$ additively decomposable functions.}
\label{alg:permtest-general}
\end{algorithm}

\paragraph{Example.}
A common example for a test statistic that requires multiple additively decomposable functions is the difference in $\fone$ scores.
Similar to accuracy,
each entry $\outan\in\Set{1,\dots,C}^\nG$ and $\nG$ is the maximum length of an entry sequence.
Let $\truep{\outan}$ and $\false{\outan}$ be the number of true positive and incorrect predictions made in entry $\outan$ respectively.
Then the difference in $\fone$ scores can be given as
\begin{equation}
    \teststatistic(\vouta, \voutb) \defeq \frac{\sum_{n=1}^{\nN}\truep{\outan}}{\sum_{n=1}^{\nN}\truep{\outan} +\frac{1}{2} \sum_{n=1}^{\nN}\false{\outan}} - \frac{\sum_{n=1}^{\nN}\truep{\outbn}}{\sum_{n=1}^{\nN}\truep{\outbn} +\frac{1}{2} \sum_{n=1}^{\nN}\false{\outbn}}
\end{equation}
We can therefore use four-additively decomposable functions, $\boldsymbol{\score_1}$ to $\boldsymbol{\score_4}$, that decompose such that $\score_1(\outan, \outbn)= \score_3(\outbn, \outan)=\truep{\outan}$ and $\score_2(\outan, \outbn)= \score_4(\outbn, \outan)=\false{\outan}$.
Our $\hscore$ function then takes four arguments and can be defined as
\begin{equation}
    \hscore(x_1, x_2, x_3, x_4) \defeq \frac{x_1}{x_1 + \frac{1}{2}x_2} - \frac{x_3}{x_3 + \frac{1}{2}x_4}
\end{equation}
We can additionally apply an absolute value to $\hscore$ to check for the absolute difference in $\fone$ scores; doing this would make the significance test two-tailed rather than one-tailed. 

\end{document}